\documentclass[sigconf]{aamas} 
\usepackage{balance} 

\setcopyright{ifaamas}
\acmConference[]{ArXiv}{2021}{March}
\copyrightyear{}
\acmYear{}
\acmDOI{}
\acmPrice{}
\acmISBN{}

\acmSubmissionID{286}
\usepackage[boxed]{algorithm}
\usepackage[noend]{algorithmic}
\usepackage{booktabs}

\usepackage{flushend}
\usepackage{tikz}

\newtheorem{remark}{Remark}%
\usepackage{nicefrac}
\usepackage{blkarray}

\usepackage{bm}
\usepackage{eucal}
\usepackage{todonotes}
\newcommand{\fahimeh}[1]{\color{magenta}{\textbf{Fahimeh says:}#1}}

\newcommand{\cI}{\mathcal{I}}

\newcommand{\cP}{\mathcal{P}}

\usepackage{enumerate}
			\usepackage{boxedminipage}
			\usepackage{xspace}

			\newcommand{\pbDef}[3]{%
			\noindent
			\begin{center}
			\begin{boxedminipage}{0.9\columnwidth}
			#1\\[5pt]
			\begin{tabular}{p{0.07\columnwidth}p{0.86\columnwidth}}
			Input: & #2\\
			{ Task}: & #3
			\end{tabular}
			\end{boxedminipage}
			\end{center}
			}

 \setcopyright{none}
\begin{document}

\setlength{\abovedisplayskip}{3pt}
\setlength{\belowdisplayskip}{3pt}

\title{Multi-Robot Task Allocation---Complexity and Approximation}


%
\author{Haris Aziz}
\affiliation{
  \institution{UNSW Sydney}
  \city{Sydney, Australia}}
\email{haris.aziz@unsw.edu.au}

\author{Hau Chan}
\affiliation{
  \institution{University of Nebraska-Lincoln}
  \city{Lincoln, NE, USA}}
\email{hchan3@unl.edu}

\author{\'{A}gnes Cseh}
\orcid{0000-0003-4991-2599}
\affiliation{%
  \institution{Hasso Plattner Institute\\ University of Potsdam}
  \city{Potsdam} 
  \state{Germany} 
  }
  \email{agnes.cseh@hpi.de}
  
  \author{Bo Li}
\affiliation{
  \institution{
  Hong Kong Polytechnic University}
  \city{Hong Kong, China}}
\email{comp-bo.li@polyu.edu.hk}

  \author{Fahimeh Ramezani}
\affiliation{
  \institution{UNSW Sydney}
  \city{Sydney, Australia}}
\email{f.ramezani@unsw.edu.au}

  \author{Chenhao Wang}
\affiliation{
  \institution{University of Nebraska-Lincoln}
  \city{Lincoln, NE, USA}}
\email{chenhwang4-c@my.cityu.edu.hk}

%
%
%
%
%
%

\begin{abstract}
Multi-robot task allocation is one of the most fundamental classes of problems in robotics and is crucial  for various real-world robotic applications such as search, rescue and area exploration. 
We consider the Single-Task robots and 
Multi-Robot tasks Instantaneous Assignment (ST-MR-IA) setting  
where each task requires at least a certain number of robots 
and each robot can work on at most one task and incurs an operational cost for each task.  
Our aim is to consider a natural computational problem of 
allocating robots to complete the maximum number of tasks subject to budget constraints. 
We consider budget constraints of three different kinds: (1) total budget, (2) task budget, and (3) robot budget. 
We provide a detailed complexity analysis including results on approximations as well as polynomial-time algorithms for the general setting and important restricted settings.
\end{abstract}


\maketitle

	
\section{Introduction}
In recent years, robots of various forms (e.g., mobile and humanoid) 
have been deployed in the real world to complete a set of tasks in applications such as search and rescue~\cite{NOT+11}, area exploration~\cite{BMF+00}, and job automation \cite{Yan:2013aa,Dias:2006aa}. Tasks, in general, have various degrees of complexity and could require \emph{coalitions} of robots to cooperate jointly in order to be completed~\cite{Tang:2007aa,DuAs19a}. This naturally leads us to the crucial question of how to \emph{effectively} allocate a set of available robots to handle a set of given tasks for a particular application.  

Many of the early works from the past decades are empirical or heuristic-or---focusing on developing and evaluating systems 
(e.g., GOFER \cite{Caloud:1990aa}, CEBOT \cite{Fukuda:1992aa}, and ASyMTRe~\cite{Fang-Tang:2005aa}) 
for multi-robot coordination \cite{GeMa04a} for completing tasks. 
It was not until the early 2000s \cite{GeMa04a} 
when a solid formal (mathematical) model and taxonomy  of \emph{multi-robot task allocation} has been considered, incorporating 
organization theory from, e.g., economics, optimization, and operation research. 
Soon after emerged a line of more theoretical and algorithmic studies, modeling variations of the multi-task allocation problem \cite{Vig:2006aa,Tang:2007aa,DuAs19a,Service:2011aa}
and highlighting their connections to other multi-agent problems, such as coalition formation~\cite{Shehory:1998aa,Rahwan:2012aa,Sandholm:1999aa}.

This paper continues that line of research. We conduct a theoretical study of the class of problems that fall into the so-called \underline{S}ingle-\underline{T}ask robots and 
\underline{M}ulti-\underline{R}obot tasks \underline{I}nstantaneous \underline{A}llocation (ST-MR-IA) category. The defining features of this setting are as follows: (1) each task has the requirement that at least a certain number of robots have to be allocated to the task in order to complete it, (2) each robot can work on at most one task and incurs an operational cost for this task, (3) the robot assignment takes place  {{in a single shot.}} 

In particular, we are interested in a natural computational problem:  allocate robots to complete the maximum number of tasks subject to cost constraints. 
There are three kinds of cost constraints.
The first one requires that the total cost used to handle the tasks does not exceed a given budget, which captures the case when we have a budget for the overall application. The second cost constraint requires the cost incurring in handling each single task does not exceed the cost budget of this task, which models the owner of each task paying for the completion of their task separately. 
The last cost constraint puts a restriction on each robot, which captures the case when robots have a fixed radius of action in which they can handle tasks.

Our optimization objective aims at maximizing the number of handled tasks subject to a given budget constraint. This is a slightly different approach to minimizing the total cost subject to handling a minimum amount of tasks. Budget constraints can dramatically change the complexity landscape of theoretical problems~\cite{BBG+11,HKT16,Jut06}. Their usage in applications is ubiquitous: e.g., the increasingly popular robotics challenges~\cite{SSV16} routinely require participants to design robots that solve tasks within a given budget, which might be expressed in terms of time, distance travelled, cost, or robot size~\cite[Section~3.3]{robocup10}.



\paragraph{Our Contributions.}
We formalize computational problems to capture the multi-criteria optimization nature of the problems and undertake a computational complexity study of them. 
One of our main contributions is a characterization of the complexity of ST-MR-IA problems with respect to the upper bound on the number of robots required to undertake a task. This requirement is not handled by well-known algorithms in the literature (see e.g., \citep{ChBrHo09a}).
Our key complexity results are summarized in Table~\ref{table:summaryresults}, where $q_j \ge 1$ represents how many robots it needs to complete task $t_j$, and $q^*=\max_{t_j\in T}q_j$ is the maximum requirement among all tasks. The results highlight how the complexity of various problems is dependent on the upper bound on the number of robots required to undertake a task.

We also present detailed analyses of the problems with respect to approximations. 
We present a simple greedy algorithm with an approximation ratio of $q^* + 1$ and we prove that unless P = NP, no polynomial-time algorithm has an approximation ratio of $(q^*)^{1-\epsilon}$ for any fixed $\epsilon>0$,
even when all tasks require the same number of robots.
However, when $q_i$'s are all constants, this approximation can be improved to $\frac{1}{2}(q^* +1)$. Moreover, in the natural case of symmetric costs (i.e., for each task, all robots are identical), we show that for our three types of cost constraints, the problem is either tractable or admits a polynomial time approximation scheme (PTAS).

	\begin{table*}[htbp]
  \centering
  \caption{Complexity and Approximability of Multi-Robot Task Allocation.} \label{table:summaryresults}

    \begin{tabular}{cccccc} 
    \hline
     Problem  & {Symmetric + Uniform} &  {Symmetric} & {Uniform ($\forall j: q_j \geq 3$)} & $q^*\le 2$ & General  \\ %
    \hline
    {\sc MaxTaskWithinTotalBudget} & {in P} & {NP-hard / PTAS } & {NP-hard } & in P & UB: $q^*+1~~$ LB: $(q^*)^{1-\epsilon}$ \\

    \hline
    {\sc MaxTaskWithinTaskBudget} & {in P} & {in P } & {NP-hard } & ? &  UB: $q^*+1~~$ LB: $(q^*)^{1-\epsilon}$ \\
    \hline
   {\sc MaxTaskWithinRobotBudget} & {in P} & {in P } & {NP-hard } & in P & UB: $q^*+1~~$ LB: $(q^*)^{1-\epsilon}$  \\
    \hline
    \end{tabular}
\end{table*}

\section{Related Work}

We start with reviewing the most relevant papers in the robotics literature and then proceed to present various mathematical problems that are suitable to model different versions of the multi-robot task allocation problem.

\paragraph{Applied approaches} \citet{GeMa04a} classify multi-task allocation problems across three axes: 
(1) single-task robots (ST) vs\ multi-task robots (MT), 
(2) single-robot tasks (SR) vs\ multi-robot tasks (MR), and 
(3) instantaneous assignment (IA) vs\ time-extended assignment (TA). 
The class of problems we study fall into the Single-Task robots 
and Multi-Robot tasks Instantaneous Assignment (ST-MR-IA) category. 

Closest to our problem setting is the work of \citet{DuAs19a}, who present a greedy heuristic to solve the ST-MR-IA problem. They build this upon the many-to-one bipartite matching problem. To ensure the perfectness of the matching, they assume that the number of robots is exactly equal to the number of robots required to complete all the tasks. 

Firstly, we observe that some claims in the paper are incorrect. The mistake originates from the false assumption that the smallest cost edge always appears in at least one minimum-cost matching, and the greedy algorithm delivers an optimal solution. A matching algorithm is greedy if it iteratively chooses one of the minimum cost edges out of the edges still available to be added to the matching. Even though greedy is an easy-to-implement solution concept, the matching it delivers might have an arbitrarily higher cost than the minimum-cost matching. 

Secondly, the assumption on the number of robots imposes a strict restriction on the problem setting. In this paper, we relax this assumption and show that the complexity of the problem changes under the relaxation. More specifically, when the number of robots is less than the total number of robots required to complete all the tasks, then natural multi-criteria optimization problems arise involving two objectives: (1) maximize the number of completed tasks and (2) minimize costs.

\citet{DUAC19b} also consider a version of the multi-robot task allocation problem, where every task is handled by a subset of mobile robots. The goal is to find an allocation that minimizes a cost function, i.e., total traveled distance. They apply a clustering approach in order to bring nearby robots together to form a coalition to do the tasks. We reflect on their work by studying the special case of location-based costs in our framework. 
 

\paragraph{Related theoretical problems}
In all three of our problems, our goal is to handle as many tasks as possible, subject to a budget constraint. It is easy to see that it is sufficient to restrict our attention to solutions that allocate each task $t_j$ either 0 or exactly $q_j$ robots, where $q_j$ is the requirement of task $t_j$, which is the minimum number of assigned robots to handle~$t_j$. In view of this observation, our problems are closely related to the \emph{$D$-matching problem}~\cite{Cor88,Lov73,Seb93}, a variant of graph factor problems~\cite{Plu07}. In an instance of the $D$-matching problem, we are given a graph $G$, and a domain of integers is assigned to each vertex. The goal is to find a subgraph $G'$ of $G$ such that every vertex has a degree in $G'$ that is contained in its domain. In our case, this domain consists of the two integers $\{0,q_j\}$ for each task $t_j$ and it is $\{0,1\}$ for each robot. 

Allocation problems can often be modeled by matchings, where the objects to assign are represented by vertices, and edges in the matching mark the assignment. Assignment costs are naturally represented by a cost function on the edges of the graph. Since ST-MR-IA deals with two inherently different categories of objects to match, a \emph{bipartite matching} model with tasks on one side and robots on the other side suits this application. The standard Hungarian algorithm for computing a minimum cost perfect matching takes $O(|V|^3)$ time, where $|V|$ is the number of vertices in the graph. For bipartite graphs and costs based on locations, a minimum cost perfect matching can be computed in $O(|V|^{2.5}\log |V|)$ time~\cite{Vaid89a}. \citet{Avis83a} discusses several heuristics for matching problems. For complete bipartite graphs with metric edge costs, the greedy algorithm provides an approximation factor of $\Theta(|V|^{\log \frac{3}{2}})$~\cite{ELW15a}.

Besides two-sidedness, another crucial feature of ST-MR-IA is the assignment of multiple robots to a single task, but only one task to a robot, which can be captured by \emph{many-to-one matchings}, a special case of $b$-matchings. Many-to-one matchings with lower and upper quotas are discussed by~\citet{ACGMM18a}. Their setting involves matching applicants to posts, where a post can only be opened if at least as many applicants are assigned to it as the post's lower quota. They establish various complexity results subject to this feasibility constraint. The key difference between their problem and ours is that they maximize the number or the total weight of assigned applicants (corresponding to our robots), while we maximize the number of handled tasks subject to a budget constraint. A crucial common feature of the two problems is that the inherently challenging task is to identify the posts to open, or, in our terminology, the tasks to handle. Notice that this question only arises if there are fewer robots than total requirements on the tasks' side, which is excluded in the setting of \citet{DuAs19a}.

\citet{KT13a} consider a variant of the \emph{generalized assignment problem with minimum quantities}. In their problem, we are given $m$ bins and $n$ items such that each item has a size $s_{ij}$ and  profit $p_{ij}$ when packed to bin $j$, and each bin has some lower and upper capacity. Each bin is required to be empty or packed above its given lower capacity. The goal is to find a feasible packing of the items into the bins that maximizes the total profit. For a special subclass of the instances, the feasibility of a packing coincides with feasibility for our assignment problem. This is the case when all items have unit profit in each bin, size $s_{ij}$ is equal to our cost $c_{ij}$, and the lower and upper capacity of each bin equals our task requirement~$q_j$. However, the objective function in the bin packing problem tallies the assigned robots instead of the handled tasks, as in our problems.

For the same problem, \citet{B15a} considers bin independent profits, which means that the profit of item $i$ is uniform for all the bins, (i.e.,  $p_{ij}=p_{i}$ for every bin $j$), and item independent profits, which means that all items have the same profit for each bin $j$, (i.e., $p_{ij}=p_j$ for every item $i$). He presents a polynomial algorithm for the bin independent profits problem and proves NP-hardness for the item independent problem.

Another related problem is the \emph{2-dimensional knapsack problem}. In this problem, we are given a 2-dimensional knapsack and a set of 2-dimensional items. Each item has a 2-dimensional size vector and a profit. The goal is to select a subset of the items of maximum total profit such that the sum of their size vectors is bounded by the knapsack capacity in each dimension. \citet{kulik2010there}  prove that the 2-dimensional knapsack problem is NP-hard, even if all items have unit profit. We show that when the cost of every task is the same for each robot, one of our problems is equivalent to this problem.

\section{Model and Problems}

We consider the multi-robot task allocation problem $(R,T,q,C)$ that involves a set of robots $R=\{r_1,\ldots, r_n\}$, a set of tasks  $T=\{t_1,\ldots, t_m\}$, 
a vector of robot requirements $q=(q_1,\ldots, q_m)$, 
which specifies that task $t_j$ requires $q_j$ robots, 
and a non-negative cost matrix $C=\{c_{ij} \mid r_i\in R, t_j\in T\}$ 
that denotes the cost $c_{ij}$ for robot $r_i$ to do task $t_j$. 
We assume every robot can perform every task, and we allocate at most one task to every robot as in the ST setting (e.g., see~\cite{GeMa04a}).
A task $t_j$ is \emph{handled} if it is allocated at least
 $q_j$ robots. 
If $t_j$ is allocated less than $q_j$ robots, then it is not handled. {{Our IA model captures an offline scenario, or a single decision in an online scenario, where all currently available robots can be assigned to complete the currently open tasks, without considering possible later events or a time component at all. }} 

	%

In this paper we concentrate on three main problems. In each of these problems, a budget constraint $W \in \mathbb{Z}^+$ is given, and subject to this, the number of handled tasks is maximized. Notice that our objective function does not directly contain the cost of the assignment---this is encapsulated exclusively in the budget constraint.

\begin{enumerate}
    \item In our first problem, we have a total budget $W$, and we want to know the maximum number of tasks one can handle so that the cost of the allocation does not exceed~$W$.   


\pbDef{\textsc{MaxTasksWithinTotalBudget}}{$(R,T,q,C)$, and $W$.} { Compute an allocation that handles the maximum number of tasks subject to having a total cost at most~$W$.}

\item We define the cost of a handled task $t_j$ to be the sum of the $q_j$ robots' costs that are allocated to it. 
In our second problem, every task has a limited budget~$W$. As such, task $t_j$ is handled if it is allocated at least $q_j$ robots and its cost is not more than~$W$. 

\pbDef{\textsc{MaxTasksWithinTaskBudget}}{$(R,T,q,C)$, and $W$.}
{	Compute an allocation that handles the maximum number of tasks subject to incurring a cost of at most $W$ for each handled task.}
 
 \item Our third problem is symmetric to \textsc{MaxTasksWithinTaskBudget}. Instead of limiting the cost of handling each task, we restrict the cost incurred to be at most $W$ for  each robot.  

\pbDef{\textsc{MaxTasksWithinRobotBudget}}{$(R,T,q,C)$, and $W$.}{ Compute an allocation that handles the maximum number of tasks subject to incurring a cost of at most $W$ for each robot.}

We remark that \textsc{MaxTasksWithinRobotBudget} reduces to \textsc{MaxTaskWithinTotalBudget} if we assign a sufficiently large cost, say, $nW+1$, to those edges that incur a higher cost than the robot budget~$W$. Then, we set the total cost for the created \textsc{MaxTaskWithinTotalBudget} problem to $nW$, so that the only restriction it enforces is that these edges cannot occur in a solution.
\end{enumerate}

{{Notice that all three kinds of budget are practically able to lift the condition on all robots being able to handle all tasks. If there are no restrictions on the cost function on edges, then robots that are unable to handle a task can declare a cost larger than the budget $W$, which ensures that they will not be assigned to this particular task.}}

Next, we formalize our problems using integer linear programming.
Let $x_{ij} \in \{0,1\}$ denote whether robot $r_i$ is assigned to task $t_j$,
and $z_j\in \{0,1\}$ denote whether task $t_j$ is handled.
Then \textsc{MaxTasksWithinTotalBudget} can be modeled the following integer linear program~$\cI\cP_{1}$.


\begin{equation} \tag{$\cI\cP_{1}$}
\begin{array}{ll@{}rll}
\text{max}  & \displaystyle\sum\limits_{j=1}^{m} &z_{j} &&\\
\text{s.t.}& \displaystyle\sum\limits_{i=1}^{n} \displaystyle\sum\limits_{j=1}^{m} &c_{ij}x_{ij} &\leq W& \\
                 & \displaystyle\sum\limits_{i=1}^{n} & x_{ij} &\geq q_j z_j&  \forall j\in [m] \\
                 & \displaystyle\sum\limits_{j=1}^{m} & x_{ij} &\leq 1&  \forall i\in [n] \\
                 &  & x_{ij} &\in \{0,1\}&  \forall i\in [n], \forall j\in [m] \\
                 &  & z_{j} &\in \{0,1\}&   \forall j\in [m] \\
\end{array}
\end{equation}

The first constraint in $\cI\cP_{1}$ guarantees that the total cost does not exceed budget $W$.
The second constraint ensures that every handled task $t_j$ (with $z_j = 1$) is assigned at least $q_j$ robots. The third constraint expresses that each robot can perform at most one task. 
Regarding \textsc{MaxTasksWithinTaskBudget}, we can define $\cI\cP_{2}$ similarly to $\cI\cP_{1}$ where the first constraint is replaced by the following:
\[
\sum_{i=1}^n c_{ij}x_{ij} \le Wz_j, \quad \forall j.
\]
That is, the cost incurred by each handled task $t_j$ cannot exceed~$W$.
Finally, for \textsc{MaxTasksWithinRobotBudget}, we can define $\cI\cP_{3}$ and the only difference with $\cI\cP_{1}$ is that the first constraint is replaced by:
\[
\sum_{j=1}^m c_{ij}x_{ij} \le W, \quad \forall i \in [n],
\]
which means the cost incurred to each robot $r_i$ cannot exceed $W$. Accordingly, for small-size instances, we can use commonly used solvers to solve these integer programs \cite{Conforti:2014aa}.

\paragraph{Restrictions} Each of our three main problems can be adapted to various restrictions of the input. We consider three restrictions, each of which captures a specific feature of an application.
\begin{itemize}
    \item We say that the costs are \emph{location-based} if $c_{ij}$ is equal to the distance between $r_i\in R$ and $t_j\in T$, where both $r_i$ and $t_j$ have physical locations on the plane. 
Location-based costs~\citep{LMK+05a,DuAs19a} are especially meaningful when considering multi-robot routing problems in which robots and tasks have their locations and the total costs are proportional to the distance travelled or fuel consumed.
\item Another restriction is when tasks have \emph{uniform requirements}: $q_j=q$ for all $j\in [m]$ (see, e.g., \citep{ACGMM18a,DuAs19a,DUAC19b}). This constraint means that each task in the instance needs the same amount of robots to be handled. Notice that the tasks are far from identical, because the cost of assigning a robot to a task differs for each task-robot pair.
\item Uniform costs on the tasks' side are called \emph{symmetric costs}: $c_{ij}=c_{kj}=c_j$ for all $r_i,r_k\in R$ and $j\in [m]$ (see e.g., \citep{DUAC19b}). This translates into paying the same price after each robot on a specific task. Notice that this set price differs for different tasks.
\end{itemize}
 


\section{Tractable problems}
\label{se:tractable}
In this section, we present four problem variants that can be solved in polynomial time. Theorem~\ref{min} proves that one can compute the lowest cost at which all tasks can be handled, if any such assignment exists. After this we study the case of uniform requirements and symmetric costs for {\sc MaxTaskWithinTotalBudget} in Theorem~\ref{th:ursc}, and then symmetric costs with and without uniform requirements for the other two problems in Theorem~\ref{th:ur}. We finish with the case of requirements at most 2 in Theorem~\ref{th:r2}.



	%

%
\begin{theorem}\label{min}
If $n\geq \sum_{j\in [m]}q_j$, then the allocation that handles all tasks at a minimum total cost 
can be computed in polynomial time.
\end{theorem}

\begin{proof}
The problem can be modeled as a minimum-weight perfect $b$-matching problem in a bipartite graph, which can be solved in polynomial time~\cite[Corollary 31.4a.]{Sch03}. The graph is constructed as follows. One side consists of $r$ robot vertices, with $b(r) = 1$. The $m$ vertices on the other side represent tasks, each of them having capacity $b(t_j) = q_j$. To ensure the perfectness of the matching, we introduce a dummy task vertex $t$ with $b(t) = n - \sum_{j\in [m]}q_j$ and connect it to all robot vertices via a 0-weight edge.

Each perfect $b$-matching of weight $w$ corresponds to a robot assignment of cost $w$ handling all tasks and vice versa. 
\end{proof}
		
Theorem~\ref{min} highlights the importance of a budget---over the total cost, the cost at each task, or at each robot. Without this budget restriction, finding a minimum cost assignment handling all tasks is not challenging. We now turn to our budget-restricted problems, and start with observing a polynomial-time solvable case when two of our restrictions, uniform requirements and symmetric costs are both present. 		

\begin{theorem}\label{th:ursc}
For uniform requirements and symmetric costs, {\sc MaxTaskWithinTotalBudget} can be solved in $O(m\log m)$ time.
\end{theorem}
\begin{proof}
Handing task $t_j$ incurs cost $qc_j$, and our goal is to handle as many tasks as possible without exceeding our total budget~$W$. To this end, we first sort and reindex the tasks in increasing order of their cost~$c_j$, which takes $O(m\log m)$ time. 
Then, we greedily allocate robots to the tasks starting from $t_1$ until no more task can be fully handled.		
\end{proof}

			 

			 
		
%
		
\begin{theorem}\label{th:ur}
For  symmetric costs, {\sc MaxTaskWithinTaskBudget} and {\sc MaxTaskWithinRobotBudget} can be solved in $O(m\log{m})$ time. If uniform requirements are also present, the running times reduce to~$O(m)$.
\end{theorem}

\begin{proof}
In {\sc MaxTaskWithinTaskBudget}, each task has a budget of $W$, and thus can be handled if and only if $q_j c_j\leq W$ and there are still $q_j$ unallocated robots. In order to maximize the number of handled tasks, we first need to order and reindex them in an increasing order of their requirements~$q_j$, which takes $O(m\log{m})$ time. It is enough to check $q_j c_j\leq W$ for each task greedily in this order and assign arbitrary $q_j$ robots to it in case it is fulfilled and there are enough unassigned robots left. If the requirements are uniform, we can omit the ordering step, which reduces the running time to~$O(m)$.

A task $t_j$ can only be handled by a set of robots in {\sc MaxTaskWithinRobotBudget} if $c_j \leq W$ for its symmetric cost~$c_j$. All tasks not fulfilling this inequality can be deleted. For the remainder of the graph, no budget restriction applies, and thus we can run our algorithm for {\sc MaxTaskWithinTaskBudget} with an artificial task budget $\max_{j \in[m]}{q_j c_j}$.
\end{proof}
				



Next, we assume every task requires at most two robots and show {{ in Theorem~\ref{th:r2}}} that \textsc{MaxTaskWithinTotalBudget} and its special case, \textsc{MaxTaskWithinRobotBudget} can be solved in polynomial time. {{Our proof relies on the following statement.}}

{\begin{lemma}
\label{le:2}
If each task requires at most 2 robots, then finding a minimum total budget $W$ under which $k$ tasks can be handled can be solved in time $O(k(mn+(m+n)\log(m+n)))$.
\end{lemma}
\begin{proof}
The constrained maximum weighted 2-path-packing of size $k$ problem on bipartite graphs ($k$-\textsc{cmwpb})~\citep{FWC14} is defined on a bipartite graph with positive edge weights and vertex sets $A$ and~$B$. The task is to pack exactly $k$ non-overlapping 2-paths so that vertices in $A$ have degree 0 or 2, while vertices in $B$ have degree 0 or 1 in the solution. 
The optimal solution is of maximum weight subject to this feasibility requirement.
This problem can be solved in time $O(k(mn+(m+n)\log(m+n)))$, where $|A|=m$ and $|B|=n$. We convert each instance of \textsc{MaxTaskWithinTotalBudget} into an instance of this problem.

Each task will be represented by a vertex in $A$, and each robot by a vertex in $B$. Besides these vertices, we also introduce a dummy robot $r_j$ to each task $t_j$ with $q_j = 1$, and set the edge weight $w(r_j,t_j)$ to a number $N$ that is larger than the maximum edge cost in our input. The weight function on the edge connecting $r_i$ to $t_j$ is defined as $w(r_i,t_j) = N-c_{ij} > 0$. 
These modifications might enlarge the set of vertices compared to $m+n$ and the set of edges compared to $mn$ by a constant factor at most.

We now show that each optimal solution for $k$-\textsc{cmwpb} corresponds to a feasible robot allocation handling $k$ tasks, while each feasible robot allocation handling $k$ tasks corresponds to a not necessarily optimal solution for $k$-\textsc{cmwpb}. Each 2-path is a feasible solution for $k$-\textsc{cmwpb} by definition and corresponds to a handled task with exactly 2 assigned robots. 
In the case of a handled unit-requirement task, then we can assume without loss of generality that 
the maximum weight solution includes the $N$-cost dummy edge, since it is the largest edge weight in the graph and the dummy robot has no further edge. 
This guarantees that unit-requirement tasks are indeed assigned exactly one robot from the original set of robots, if any. Due to the non-overlapping condition, each robot is assigned to one task at most. Similarly, each robot assignment handling $k$ tasks readily translates into a solution for $k$-\textsc{cmwpb} if we assign each handled unit-capacity task its dummy robot as well.

If the cost of the robot allocation is denoted by $C$, then the weight of the corresponding solution for $k$-\textsc{cmwpb} is $2kN-C$, because $w(r_i,t_j) = N-c_{ij}$ and each handled task adds exactly 2 edges to the 2-path packing. 
This weight function is obviously maximized if and only if $C$ is minimized. By computing an optimal solution for $k$-\textsc{cmwpb}, we can thus output the cost $W$ of the minimum cost assignment handling $k$ tasks.
\end{proof}

\begin{theorem}\label{th:r2}
If each task requires at most 2 robots, \textsc{MaxTaskWithinTotalBudget} and \textsc{MaxTaskWithinRobotBudget} can be solved in time $O(m\log{m}(mn+(m+n)\log(m+n)))$.
\end{theorem}
\begin{proof}
We run binary search on $1, 2, \ldots, m$ for the largest $k$ value for which the cost of the minimum cost assignment handling exactly $k$ tasks is below the given~$W$. We run the algorithm from Lemma~\ref{le:2} to compute the minimum cost for each~$k$.
\end{proof}
}


\section{Hard problems}

In this section, we prove several computational intractability results. In Section~\ref{se:tractable}, we have shown polynomial time solvability for 2 of our problems if symmetric costs are present. To complement this result, in Theorem~\ref{thm:sssp} we prove hardness for our third problem, {\sc MaxTaskWithinTotalBudget}, under the same assumption. Then, we show in Theorem~\ref{thm:hardness_merged} that uniform requirements alone do not simplify the problem, and we finish with proving in Theorem~\ref{th:loc} that {\sc MaxTaskWithinTotalBudget} is NP-hard even for location-based costs.

To motivate the hardness result in Theorem~\ref{thm:sssp}, we first demonstrate through two examples that greedy algorithms do not deliver an optimal solution for {\sc MaxTaskWithinTotalBudget} with symmetric costs. We start with the greedy algorithm  that sorts the tasks in increasing order of their requirements. 

\begin{example}
Consider $T=\{t_1, t_2, t_3\}$ with $q=(1, 2, 2)$ and $4$ robots. With symmetric costs $c_{1}=100, c_{2}=1, c_{3}=1$ and total budget $W=100$, the greedy algorithm, based on sorting the tasks in increasing order of their requirement $q_j$, chooses task $t_1$ and obtains a value of 1. 
However, the optimum solution is 2 (i.e., $t_2$ and $t_3$ are handled). 
\end{example}


Another counterexample shows suboptimality for the greedy algorithm that sorts tasks based on $q_j c_{j}$.

\begin{example}
Let $T=\{t_1, t_2, t_3\}$ with $q=(100, 2, 2)$ and we have $100$ robots with symmetric cost $c_1=1, c_2=60, c_3=60$. If our total budget is $W=250$, then the greedy algorithm, based on $q_j c_{j}$, chooses task $t_1$ and allocate all 100 robots to it. However, the optimal solution is to handle tasks $t_2$ and~$t_3$. 
\end{example}

More generally, under symmetric costs, we can show that {\sc MaxTaskWithinTotalBudget}, unlike {\sc MaxTaskWithinTaskBudget} and {\sc MaxTaskWithinRobotBudget}, is NP-hard.

\begin{theorem}\label{thm:sssp}
For symmetric costs, {\sc MaxTaskWithinTotalBudget} is NP-hard.
\end{theorem}
\begin{proof}
Consider  the \emph{2-dimensional Knapsack Problem} (2KP): given a $2$-dimensional bin, and a set of items, each having a $2$-dimensional size vector and a profit, the goal is to select a subset of the items of maximum total profit such that the sum of their size vectors is bounded by the bin capacity in each of the two dimensions. Kulik and  Shachnai~\cite{kulik2010there} prove that the 2KP is NP-hard, even when all items have unit profit. We show that {\sc MaxTaskWithinTotalBudget} with symmetric costs is equivalent to the 2KP with unit profit, and thus is NP-hard.

For any instance $(R,T,q,C)$ of {\sc MaxTaskWithinTotalBudget} with symmetric costs, let $x_j$ be a variable indicating whether to handle task $t_j\in T$. The LP formulation is given by 
\begin{align}
\text{max}~~~~~&\sum_{j=1}^mx_j   & \nonumber\\
\text{s.t.}~~~~&\sum_{j=1}^mq_jx_j  \le n, &  \label{eq:1}\\
&\sum_{j=1}^mq_jc_jx_j \le W,  & \label{eq:2}\\
& x_j\in\{0,1\},& t_j\in T.\nonumber
\end{align}
Constraint (\ref{eq:1}) says that the total number of robots used to handle tasks is no more than $|R|=n$. Constraint (\ref{eq:2}) says that the total cost is no more than the budget~$W$. It is easy to see that the above program is an instance of 2KP with unit profit and integer coefficients. 

For the other direction, given an instance of 2KP with unit profit in a decision version, suppose the inputs are rational. It is a ``yes" instance if and only if the version that we construct by multiplying all capacities and sizes with the least common multiplicator is a ``yes" instance. Then the constructed version has integer capacities and sizes, and can be formulated by the above program. Therefore, it is an instance of {\sc MaxTaskWithinTotalBudget} with symmetric costs.
\end{proof}

In Section~\ref{sec:improved-approx} we will show that there is a PTAS for {\sc MaxTaskWithinTotalBudget} with symmetric costs. Next, we drop the symmetry requirement on costs and prove computational hardness results for our three problems in the case of uniform requirements. First we restate Theorem 1 from~\cite{ACGMM18a} in our terminology.

\begin{theorem}[Arulselvan et al.~\cite{ACGMM18a}]\label{thm:matchings}
In an instance where each task has a uniform requirement of 3 robots and each $c_{ij}$ is either 1 or a large $N>W$, finding a minimum cost allocation that handles at least $k$ tasks for a given $k$ is NP-hard.
\end{theorem}

We utilize this result to show hardness for our three problems.

\begin{theorem}\label{thm:hardness_merged}
Each of \textsc{MaxTasksWithinTotalBudget}, \textsc{MaxTasksWithinTaskBudget}, and \textsc{MaxTasksWithinRobotBudget} is NP-hard,
 even if each task has a uniform requirement of 3 robots and each $c_{ij}$ is either 1 or a large $N>W$. 
\end{theorem}

\begin{proof}
 For each instance satisfying the requirements of Theorem~\ref{thm:matchings}, we set a different budget~$W$, as follows. 
 \begin{itemize}
     \item $W = 3m$ for \textsc{MaxTasksWithinTotalBudget}
     \item $W = 3$ for \textsc{MaxTasksWithinTaskBudget}
     \item $W = 1$ for \textsc{MaxTasksWithinRobotBudget}
\end{itemize}
Each of these budgets ensures that a robot can only be assigned to a task if it costs~1. All other assumptions are analogous in our three problems (e.g., costs, requirements, and $k$) and the problem in Theorem~\ref{thm:matchings}. 
 Therefore, a solution to any of our problems is also a solution for the problem in Theorem~\ref{thm:matchings} and vice versa.
\end{proof}


Finally, we show the hardness of \textsc{MaxTasksWithinTaskBudget} for a strict subcase of location-based costs.


\begin{theorem}
\label{th:loc}
\textsc{MaxTasksWithinTaskBudget} is NP-hard for location-based costs, even if $c_{ij_1} = c_{ij_2}$ for each robot $r_i$ and tasks $t_{j_1}, t_{j_2}$, and each task has a uniform requirement of 3 robots.
\end{theorem}
\begin{proof}
We will show that deciding whether there exists an assignment handling all tasks below cost $W$ at each task is NP-complete. We reduce the strongly NP-complete 3-partition problem~\cite{GaJo79a} to our problem. The input to this problem is a multiset $S=s_1, s_2, \ldots, s_{3k}$ of positive integers. The output is whether or not there exists a partition of $S$ into $k$ triplets such that the sum of the numbers in each triplet is equal to $T = \frac{s_1 + s_2 + \ldots + s_{3k}}{k}$. 

Our instance will have $k$ tasks of requirement 3 each, $3k$ robots, and a task budget $W=T$. To each integer $s_i \in S$ we create a robot $r_i$ and assign cost $c_{ij}=s_i$ for all tasks~$t_j$. Each 3-partitioning of $S$ trivially translates into an assignment, where a handled task represents a set of 3 integers with total cost $W$ and vice versa. Notice that since $kT = s_1 + s_2 + \ldots + s_{3k}$, an assignment staying below the task budget $T$ at each vertex must reach $T$ at each task.
\end{proof}

\section{Approximation}
In this section, we consider approximation algorithms for our problems. We show that a simple greedy algorithm can provide a tight approximation of $\max_i q_i + 1$ for all problems. 
When $q_i$'s are all constants, this approximation can be improved to $\frac{1}{2}(\max_i q_i +1)$. 
For {\sc MaxTaskWithinTotalBudget}, when costs are symmetric, there is a polynomial time approximation scheme (PTAS). 

We say an algorithm $\mathcal A$ has approximation ratio $\alpha\ge1$ for a problem, if for any instance $I$ of this problem, the number of tasks $\mathcal A(I)$ handled by the algorithm times $\alpha$ is no less than the maximum possible number of tasks $opt(I)$ that can be handled, i.e., $\alpha\cdot \mathcal A(I)\ge opt(I)$.

\subsection{A Greedy Algorithm}

We first present a simple greedy algorithm for the objective of {\sc MaxTaskWithinTotalBudget}, as shown in Algorithm~\ref{alg:greedy}, and then we show how to adjust it for the other two problems. 

\begin{algorithm}[t]
\caption{\hspace{-2pt}{ A Greedy Algorithm}}
\label{alg:greedy}
\begin{algorithmic}[1]
\REQUIRE A multi-robot task allocation problem $(R,T,q,C)$.
\ENSURE An assignment $A = \{(t_1, R_1), \cdots, (t_k, R_k)\}$
\STATE Initially, set $A = \emptyset$.

\WHILE{$T \neq \emptyset$}
\FOR{$t_j \in T$}
\STATE Set
$
R_j = \begin{cases}
\arg\min\limits_{R'\subseteq R: |R'| = q_j} \sum_{r_i \in R'} c_{ij} & \text{ if $|R| \ge q_j$}\\
\emptyset & \text{ otherwise.} 
\end{cases}
$

\STATE Set
$
c(t_j, R_j) = \begin{cases}
\sum_{r_i \in R_j} c_{ij} & \text{ if $R_j \neq \emptyset$}\\
+\infty & \text{ otherwise.}
\end{cases}
$
\ENDFOR
\STATE Set \begin{align*}
(t_j^*, R^*_j) = \arg\min_{t_j \in T} c(t_j, R_j),
\end{align*}
where ties are broken arbitrarily.
\IF{$\sum_{(t_j,R_j) \in A} \sum_{r_i\in R_j} c_{ij} + c(t_j^*, R^*_j) \le W$} \label{step:condition}
\STATE $A = A \cup \{(t_j^*, R_j^*)\}$.
\STATE $T = T \setminus \{t_j^*\}$ and $R = R \setminus R^*_j$.
\ELSE
\STATE \textbf{break}
\ENDIF




\ENDWHILE


\RETURN $A$
\end{algorithmic}
\end{algorithm}
       


\begin{theorem}\label{th:boapprox}
For {\sc MaxTaskWithinTotalBudget}, Greedy algorithm can be implemented in polynomial time and is $(q^*+1)$-approximation, where $q^* = \max_{j \in T} q_j$.
\end{theorem}

\begin{proof}
We first show that the greedy algorithm can be implemented in polynomial time. 
There are at most $m$ rounds for the {\bf while} loop because for each round, either the size of $T$ decreases by 1 or the algorithm terminates. 
Within each round of the {\bf while} loop, it suffices to check that the {\bf for} loop can be done in polynomial time. 
This is because to find the lowest completing cost for each task $t_i$, we only need to sum the lowest $q_j$ costs between $t_j$ and all remaining robots. 


Next, we prove the approximation ratio.
Suppose the solution of the greedy algorithm is $A = \{(t_1, R_1), \cdots, (t_k, R_k)\}$ where $k = |A|$ is the number of tasks that are handled, and each $(t_j, R_j) \in T \times 2^{R}$ represents task $j$ is handled by robots in $R_j$ in solution $A$. 
Similarly, let an optimal solution be $O = \{(t_{j_1}, R_{j_1}'), \cdots, (t_{j_l}, R_{j_l}')\}$.
In the following, we prove that $k \ge \frac{l}{q^*+1}$.

To see the difference between $A$ and $O$, we observe that, 
for each $1 \le j \le k$, $|R_j| = q_j \le q^*$,
then task $t_j$ can impede at most $q^*+1$ tasks in $O$ including $t_j$ itself. Formally, let $O_0 = \emptyset$, and for each $t_j \in A$, we recursively define $O_j \subseteq O \setminus \cup_{f \le j-1} O_f$ as the tasks in $O$ that are not able to be handled because of $t_j$. 
That is 
\[
O_j = \{ (t_{j'}, R_{j'}') \in O \setminus \cup_{f \le j-1} O_f \mid R_{j'}' \cap R_j \neq \emptyset \}.
\]
Note that by our definition, tasks in $O_j$ cannot be impeded by 
$\{t_1, \cdots, t_{j-1}\}$, i.e., $R_{j'}' \cap (\cup_{f < j} R_{f}) = \emptyset$ for all $(t_{j'}, R_{j'}') \in O_j$.
Moreover, $|O_j| \le q^* + 1$. 
We further partition $A$ into $A_1$ and $A_2$, where
$A_1$ is the set of tasks whose selection impedes tasks in $O$, i.e., 
\[
A_1 = \{(t_j, R_j) \in A \mid |O_j| \ge 1\}
\]
and $A_2 = A\setminus A_1$. 
Accordingly, we claim the following.
\begin{align}\label{eq:(1)}
    (q^*+1) \cdot |A_1| &\ge  |\cup_{(t_j,R_j)\in A_1} O_j| 
\end{align}
and
\begin{align} \label{eq:(2)}
    \sum_{z \in R_j} c_{zj} &\le \sum_{z \in R_{j'}'} c_{zj'},
\end{align}
for each $(t_j, R_j) \in A_1$ and any $(t_{j'}, R_{j'}') \in O_j$.
	   Inequality (\ref{eq:(1)}) holds because $|O_j| \le q^*+1$ for any $j$.
	   Inequality (\ref{eq:(2)}) holds because the algorithm always selects a feasible task-robot pair with the smallest cost. Specifically, if $\sum_{z \in R_j} c_{zj} > \sum_{z \in R_{j'}'} c_{zj'}$, since $(t_{j'}, R_{j'}')$ is not impeded by $\{t_1, \cdots, t_{j-1}\}$, then $(t_{j'}, R_{j'}')$ should be added to $A$ before $(t_j, R_j)$, a contradiction. 
	   Let $O' = O \setminus \cup_{(t_j,R_j)\in A_1} O_j$. 
	   Next, we show that 
	   \begin{align}\label{eq:(3)}
	       |A_2| \ge |O'|.
	   \end{align}
	   Before we prove Inequality (\ref{eq:(3)}), we observe that combing the definition of $A_1$ and Inequality (\ref{eq:(2)}), we have the following: 
	   \[
	   \sum_{(t_j, R_j) \in A_1} \sum_{z \in R_j} c_{zj} \le \sum_{(t_j, R_j) \in A_1} \sum_{(j',R_{j'}') \in O_j} \sum_{z \in R_{j'}'} c_{zj'},
	   \]
	   which means the available cost budget for $A_2$ is no smaller than for $O'$.
	   Moreover, as the greedy algorithm always selects the pair with the smallest cost, 
	   for any $(t_j, R_j)\in A_2$ and any $(t_{j'}, R_{j'}') \in O'$, $\sum_{z \in R_j} c_{zj} \le \sum_{z \in R_{j'}'} c_{zj'}$. Otherwise, $(t_{j'}, R_{j'}') \in O'$ should be added to $A$ before $(t_j, R_j)\in A_2$, recalling that $O'$ are not impeded by $A$.
	   Thus $|A_2| \ge |O'|$.

	   Combing Inequalities (\ref{eq:(1)}) and (\ref{eq:(3)}), 
	   \[
	   k = |A_1| + |A_2| \ge \frac{|\cup_{(t_j,R_j)\in A_1} O_j|}{q^*+1} + |O'| \ge \frac{l}{q^*+1},
	   \]
	   which completes the proof of Theorem \ref{th:boapprox}.
	   \end{proof}

	   \begin{corollary}\label{coro:boapprox}
	   For {\sc MaxTasksWithinTaskBudget} and {\sc MaxTaskWithinRobotBudget}, there are Greedy algorithms that are $(q^*+1)$-approximation, where $q^* = \max_{j \in T} q_j$.
	   \end{corollary}
	   \begin{proof}
	   Regarding {\sc MaxTasksWithinTaskBudget}, we modify Algorithm \ref{alg:greedy} by exchanging the condition in Step \ref{step:condition} with
	   \[
	   c(t_j^*, R^*_j) \le W.
	   \]
	   Then, we note that both Inequalities (\ref{eq:(1)}) and (\ref{eq:(2)}) in the proof of Theorem \ref{th:boapprox} hold for the same reason. Therefore, it suffices to prove Inequality (\ref{eq:(3)}), which holds more strongly than under the objective of {\sc MaxTaskWithinTotalBudget}. Actually, it must be that $O' = \emptyset$; otherwise, there is a feasible pair $(j, R'_j) \in O'$ that is not impeded by any pair in $A$, then it should be selected by the algorithm, which is a contradiction. 
	   
	   Regarding {\sc MaxTaskWithinRobotBudget},  we can first set all the robot-task cost that are greater than $W$ to $mn\max_{t_i,r_j} c_{ij}$ and then use Algorithm \ref{alg:greedy} with budget $mn\max_{t_i,r_j} c_{ij}$.
	   \end{proof}
	   
	   We note that an alternative way to obtain the approximation result for {\sc MaxTaskWithinRobotBudget} is to reduce our problem to the {\em maximum weight many-to-one matching with lower and upper quotas} defined in \citet{ACGMM18a}.
	   Roughly, in the induced bipartite graph between tasks and robots, all edges with costs more than $W$ are deleted, the lower and upper quotas for each task $t_j$ are set to be $q_j$, and edge weight between each robot and task $t_j$ is set to $\frac{1}{q_j}$. The quotas ensure that at least $q_j$ robots to be assigned to task $t_j$ if any, while the total weights of a handled task is 1 and thus the weight of an assignment is equal to the number of handled tasks. However, we note that this approach cannot be generalized to deal with the other two objectives {\sc MaxTaskWithinTotalBudget} and  {\sc MaxTasksWithinTaskBudget}.
	   
	   Next, we prove that, via the following theorem, the approximation ratio in Theorem \ref{th:boapprox} (as well as in Corollary \ref{coro:boapprox}) for the general case is almost the best we are able to guarantee.

\begin{theorem}
\label{the:hard:general}
Even with uniform requirements, {\sc MaxTaskWithinTotalBudget},  {\sc MaxTasksWithinTaskBudget} and {\sc MaxTaskWithinRobotBudget} cannot be approximated within a factor of $({q^*})^{1-\epsilon}$ for any fixed $\epsilon>0$  by any polynomial-time algorithm, where $q^* = \max_{j \in T} q_j$, unless P = NP.
\end{theorem}


\begin{proof}
We first consider {\sc MaxTaskWithinTotalBudget}.
We reduce from Independent Set (IS) Problem~\cite{hastad1999clique}.
Given an IS problem $G=(V,E)$ with $|V|$ nodes and $|E|$ edges, we construct the following instance with $|V|$ tasks, $|V|^2 - |E|$ robots, and budget $W=0$.
	   
For each $i \in V$, we construct a task $t_i$ with $q_i = |V|$ and $n$ robots $i_1, \cdots, i_{|V|}$.
All robots $i_1, \cdots, i_{|V|}$ have cost 0 to $t_i$ and 1 to any other task.
Then, for each $(i,j) \in E$, we merge robots $i_j$ and $j_i$ and reset the cost for the merged robot to do both $t_i$ and $t_j$ to be 0 and the other tasks to be 1.

Note that since $W=0$, each task $t_i$ can only be handled by the robots whose cost to do $t_i$ is 0. Moreover, for $t_i$ to be handled, it requires at least $|V|$ robots.
Thus maximizing the number of handled tasks within total cost 0 is equivalent to maximizing the size of an independent set in graph $G$.
Since IS cannot be approximated better than $|V|^{1-\epsilon}$ for any $\epsilon > 0$ \cite{hastad1999clique}, our problem cannot be approximated better than $({q^*})^{1-\epsilon}$, where $q^* = \max_{j} q_j = |V|$.

Regarding the other two objectives, we can directly obtain the hardness result by changing $W$ in the previous reduction to be the budget of robot or task.
\end{proof}
	   
We note that the reduction in Theorem \ref{the:hard:general} shares the same idea with \cite[Theorem~4.3]{Service:2011aa} and \cite[Theorem~9]{ACGMM18a}, whose construction also have the property that handling $k$ tasks corresponds to finding an independent set of size~$k$ and vice versa.


	   
\subsection{Improved Approximations} \label{sec:improved-approx}

When $q_i$'s are all constants, we provide an improved $\frac{1}{2}(\max_i q_i +1)$ approximation algorithm 
for {\sc MaxTaskWithinRobotBudget} and {\sc MaxTasksWithinTaskBudget}. When costs are symmetric, we show that there is actually a polynomial time approximation scheme (PTAS) for {\sc MaxTaskWithinTotalBudget}.

Given a finite set and a list of subsets, the \emph{Maximum Set Packing} problem \cite{karp1972reducibility} asks for the maximum number of pairwise disjoint sets in the list. Formally, given  a universe $\mathcal {U}$ and a family $\mathcal {S}$ of subsets of $\mathcal {U}$, a packing is a subfamily $\mathcal {C}\subseteq \mathcal {S}$ of sets such that all sets in $\mathcal {C}$ are pairwise disjoint. The \emph{Maximum $k$-Set Packing} ($k$-SP) problem considers the case where all sets in the given family $\mathcal {S}$ are of the same size $k$. \citet{hurkens1989size} provide a $\frac{k}{2}$-approximation algorithm for $k$-SP.
We next show that this algorithm can be used to approximate two of our problems.

\begin{theorem}\label{thm:sp}
When $q^*=\max_{i\in[m]}q_i$ is less than a fixed constant, there exists a $\frac{q^*+1}{2}$-approximation algorithm that runs in polynomial time, for {\sc MaxTaskWithinRobotBudget} and {\sc MaxTasksWithinTaskBudget}.
\end{theorem}

\begin{proof}

Given an instance $(R,T,q,C,W)$, we modified it to obtain an instance $(R',T,q',C',W)$ without changing the optimal value.
First, for $i\in[m]$, add $q^*-q_i$ dummy robots to $R$, each of which has 0 cost to task $t_i$, and a sufficiently large cost to any other task $t\neq t_i$. Denote the new set of robots by $R'$, consisting of tasks in $R$ and all dummy robots. The tasks have uniform requirement $q^*$. 
It is easy to see that this modification would not change the maximum number of tasks that can be handled.

Second, let $\mathcal {U}=R'\cup T$ be the universe set, and initialize $\mathcal {S}=\emptyset$ as a family of subsets. For every task $t_j\in T$ and every $q^*$ robots (say $r_1,\ldots,r_{q^*}$) in $R'$, if these $q^*$ robots can handle task $t_j$ under the cost constraint $W$ (e.g., for {\sc MaxTaskWithinRobotBudget}, the cost of each robot going to $t_j$ is at most $W$),  add the set $\{t_j,r_1,\ldots,r_{q^*}\}$ to $\mathcal {S}$. Because $q^*$ is less than some constant, this process of constructing $\mathcal {S}$ by enumeration can be done in polynomial time.

Note that all sets in $\mathcal {S}$ are of the same size $q^*+1$. Hence, we obtain a $(q^*+1)$-SP instance $(\mathcal {U},\mathcal {S})$. It is easy to see that the maximum number of pairwise disjoint sets in $\mathcal {S}$ is equal to the maximum number of tasks that can be handled in the modified instance of {\sc MaxTaskWithinRobotBudget} or {\sc MaxTasksWithinTaskBudget}. Using the $\frac{k}{2}$-approximation algorithm for $k$-SP, we complete the proof.
\end{proof}

The above algorithm improves the approximation of Greedy Algorithm in Theorem~\ref{th:boapprox}, when $q^*$ is constant. We remark that the algorithm in Theorem \ref{thm:sp} cannot apply to {\sc MaxTaskWithinTotalBudget}, because $\mathcal {S}$ cannot be constructed efficiently nor independently because of the global constraint $W$.


In the case with symmetric costs, we can obtain an approximation better than a constant approximation for {\sc MaxTaskWithinTotalBudget}.

\begin{theorem}
For symmetric costs, {\sc MaxTaskWithinTotalBudget} has a PTAS.
\end{theorem}
\begin{proof}
As shown in the proof of Theorem \ref{thm:sssp},  {\sc MaxTaskWithinTotalBudget} with symmetric costs is equivalent to the 2-dimensional knapsack problem with unit profit. It is well-known that multi-dimensional knapsack problem has a PTAS \cite{frieze1984approximation} (see an overview in \cite{freville2004}). This immediately gives a PTAS for {\sc MaxTaskWithinTotalBudget} with symmetric costs.
\end{proof}

\section{Conclusions}
Multi-robot task allocation is hugely important in the multi-agent and robotics communities. 
In this paper, we undertook a detailed study of the computational complexity of multi-robot allocation to tasks where each task has a requirement of a certain number of robots. 
A notable open problem is the complexity of {\sc MaxTaskWithinTaskBudget} for the case when each task requires at most 2 robots.
We presented several approximation algorithms for general and important special cases as well as proved approximation lower bounds. 
For future work, it will be interesting to explore parametrized algorithms for the problems considered, with respect to natural parameters such as the target number of tasks to be handled.

\section*{Acknowledgements}
		Aziz and Ramezani are supported by  the Australian Defence Science and Technology Group (DSTG) under the project
``Auctioning for distributed multi vehicle planning'' (MYIP 9953) and by US Dept. of Air force under the project 
``Efficient and fair decentralized task allocation algorithms for autonomous vehicles'' (FA2386-20-1-4063). 
Cseh is supported by the Federal Ministry of Education and Research of Germany in the framework of KI-LAB-ITSE (project number 01IS19066). Bo Li is supported by The Hong Kong Polytechnic University (Grant NO. P0034420). 

%
%
		



\bibliographystyle{ACM-Reference-Format}  

\bibliography{newref.bib}

\end{document}